\documentclass{article} 
\usepackage{iclr2018_conference,times}
\usepackage{url}

\usepackage[utf8]{inputenc}
\usepackage{graphicx}
\usepackage[colorinlistoftodos]{todonotes}
\usepackage{natbib}

\usepackage[utf8]{inputenc} 
\usepackage[T1]{fontenc}    
\usepackage{hyperref}       
\usepackage{url}            
\usepackage{booktabs}       
\usepackage{amsfonts}       
\usepackage{nicefrac}       
\usepackage{microtype}      

\usepackage{amsmath}
\usepackage{graphicx}
\usepackage[linesnumbered,ruled]{algorithm2e}
\usepackage{multirow}
\usepackage{subcaption}
\usepackage{color}
\usepackage{comment}
\usepackage{wrapfig}
\usepackage{mathrsfs}
\usepackage{epsfig}
\usepackage{graphics}
\usepackage{epsf}
\usepackage{amsmath, amsthm, amssymb, multirow, paralist}
\usepackage{color}
\newcommand{\rM}{\mathrm{M}}

\newtheorem{theorem}{Theorem}[section]

\newtheorem{thm}{Theorem}
\newtheorem{prop}{Proposition}

\title{Robust Task Clustering for Deep and Diverse\\ Multi-Task and Few-Shot Learning}

\author{Mo Yu\thanks{Equal contributions. Corresponding authors: \texttt{yum@us.ibm.com, xiaoxiao.guo@ibm.com, jinfengy@us.ibm.com}.} \quad   Xiaoxiao Guo$^*$\quad Jinfeng Yi$^*$\quad Shiyu Chang \\ 
\textbf{Saloni Potdar}\quad \textbf{Gerald Tesauro}\quad \textbf{Haoyu Wang}\quad \textbf{Bowen Zhou}\\
 \vspace{0.02cm}\\
AI Foundations -- Learning,  IBM Research\\
IBM T. J. Watson Research Center, Yorktown Heights, NY 10598\\
\texttt{yum@us.ibm.com, xiaoxiao.guo@ibm.com, jinfengy@us.ibm.com}
}

\def \B {\mathcal{B}}
\def \x {\mathbf{x}}
\def \H {\mathcal{H}}
\def \R {\mathbb{R}}

\def \sgn {\mbox{sgn}}
\def \a {\mathbf{a}}
\def \u {\mathbf{u}}

\def \E {\textbf{E}}

\def \A {\textbf{A}}

\def \y {\mathbf{y}}

\def \S {\textbf{S}}
\def \X {\textbf{X}}
\def \Q {\textbf{Q}}
\def \Z {\textbf{Z}}

\def \Y {\textbf{Y}}
\def \P {\mathcal{P}}
\def \u {\mathbf{u}}
\def \v {\mathbf{v}}
\def \e {\mathbf{e}}
\def \U {\textbf{U}}
\def \V {\textbf{V}}
\def \S {\textbf{S}}
\def \X {\textbf{X}}
\def \Q {\textbf{Q}}
\def \Z {\textbf{Z}}
\def \U {\textbf{U}}
\def \N {\textbf{N}}
\def \F {\textbf{F}}
\def \V {\textbf{V}}
\def \P {\textbf{P}}
\def \Y {\textbf{Y}}
\def \B {\textbf{B}}
\def \H {\textbf{H}}
\def \E {\textbf{E}}
\def \A {\textbf{A}}

\def \v {\mathbf{v}}
\def \e {\mathbf{e}}

\def \rank {\mbox{rank}}

\def \robusttc {\textsc{RobustTC}}

\iclrfinalcopy 

\begin{document}

\maketitle
\begin{abstract}

We investigate task clustering for deep learning-based multi-task and few-shot learning in the settings with large numbers of tasks. Our method measures task similarities using cross-task transfer performance matrix. 
Although this matrix provides us critical information regarding similarities between tasks, the uncertain task-pairs, i.e., the ones with extremely asymmetric transfer scores, may collectively mislead clustering algorithms to output an inaccurate task-partition.
Moreover, when the number of tasks is large, generating the full transfer performance matrix can be very time consuming. To overcome these limitations, we propose a novel task clustering algorithm to estimate the similarity matrix based on the theory of matrix completion. The proposed algorithm can work on partially-observed similarity matrices based on only sampled task-pairs with reliable scores, ensuring its efficiency and robustness. Our theoretical analysis shows that under mild assumptions, the reconstructed matrix perfectly matches the underlying “true” similarity matrix with an overwhelming probability. The final task partition is computed by applying an
efficient spectral clustering algorithm to the recovered matrix. Our results show that the new task clustering method can discover task clusters that benefit both multi-task learning and few-shot learning setups for sentiment classification and dialog intent classification tasks. 
\end{abstract}

\section{Introduction}

This paper leverages knowledge distilled from a large number of learning tasks~\citep{barzilai2015convex,van2017hybrid}, or \emph{MAny Task Learning (\textbf{MATL})}, to achieve the goal of (i) improving the overall performance of all tasks, as in \emph{multi-task learning (\textbf{MTL})}; and (ii) rapid-adaptation to a new task by using previously learned knowledge, similar to \emph{few-shot learning (\textbf{FSL})} and transfer learning.
Previous work on multi-task learning and transfer learning used small numbers of related tasks (usually $\sim$10) picked by human experts.  By contrast, MATL tackles hundreds or thousands of tasks~\citep{barzilai2015convex,van2017hybrid}, with unknown relatedness between pairs of tasks, introducing new challenges such as task diversity and model inefficiency.

MATL scenarios are increasingly common in a wide range of machine learning applications
with potentially huge impact.
Examples include reinforcement learning for game playing -- where many numbers of sub-goals are treated as tasks by the agents for joint-learning, e.g. \citet{van2017hybrid} achieved the state-of-the-art on the Ms. Pac-Man game by using a multi-task learning architecture to approximate rewards of over 1,000 sub-goals (reward functions).
Another important example is enterprise AI cloud services -- where many clients submit various tasks/datasets to train machine learning models for business-specific purposes. The clients could be companies who want to know opinion from their customers on products and services, agencies that monitor public reactions to policy changes, and financial analysts who analyze news as it can potentially influence the stock-market. Such MATL-based services thus need to handle the diverse nature of clients' tasks. 

\textbf{Challenges on Handling Diverse (Heterogeneous) Tasks} \ \ 
Previous multi-task learning and few-shot learning research usually work on homogeneous tasks, e.g. all tasks are binary classification problems, or tasks are close to each other (picked by human experts) so the positive transfer between tasks is guaranteed.
However, with a large number of tasks in a MATL setting, the above assumption may not hold, i.e. we need to be able to deal with tasks with larger diversity. 
Such diversity can be reflected as (i) \textbf{tasks with varying numbers of labels}: when tasks are diverse, different tasks could have different numbers of labels; and the labels might be defined in different label spaces without relatedness. Most of the existing multi-task and few-shot learning methods will fail in this setting; and more importantly (ii)
\textbf{tasks with positive and negative transfers}: 
since tasks are not guaranteed to be similar to each other in the MATL setting, they are not always able to help each other when trained together, i.e. \emph{negative transfer} \citep{yosinski2014transferable} between tasks.
For example, in dialog services, the sentences ``\emph{What fast food do you have nearby}'' and ``\emph{Could I find any Indian food}'' may belong to two different classes ``\emph{fast\_food}'' and ``\emph{indian\_food}'' for a restaurant recommendation service in a city; while for a travel-guide service for a park, those two sentences could belong to the same class ``\emph{food\_options}''.
In this case the two tasks may hurt each other when trained jointly with a single representation function, since the first task turns to give similar representations to both sentences while the second one turns to distinguish them in the representation space.

\textbf{A Task Clustering Based Solution} \ \ 
To deal with the second challenge above, we propose to partition the tasks to clusters, making the tasks in each cluster more likely to be related. Common knowledge is only shared across tasks within a cluster, thus the negative transfer problem is alleviated.
There are a few task clustering algorithm proposed mainly for convex models \citep{kumar2012learning,kang2011learning,crammer2012learning,barzilai2015convex}, but they assume that the tasks have the same number of labels (usually binary classification). 
In order to handle tasks with varying numbers of labels, we adopt a similarity-based task clustering algorithm. The task similarity is measured by cross-task transfer performance, which is a matrix $\S$ whose $(i,j)$-entry $\S_{ij}$ is the estimated accuracy by adapting the learned representations on the $i$-th (source) task to the $j$-th (target) task.
The above task similarity computation does not require the source task and target task to have the same set of labels, as a result, our clustering algorithm could naturally handle tasks with varying numbers of labels.

Although cross-task transfer performance can provide critical information of task similarities, directly using it for task clustering may suffer from both efficiency and accuracy issues. First and most importantly, evaluation of all entries in the matrix $\S$ involves conducting the source-target transfer learning $O(n^2)$ times, where $n$ is the number of tasks. For a large number of diverse tasks where the $n$ can be larger than 1,000, evaluation of the full matrix is unacceptable (over 1M entries to evaluate). Second, the estimated cross-task performance (i.e. some $\S_{ij}$ or $\S_{ji}$ scores) is often unreliable due to small data size or label noises. When the number of the uncertain values is large, they can collectively mislead the clustering algorithm to output an incorrect task-partition. 

To address the aforementioned challenges, we propose a novel task clustering algorithm based on the theory of matrix completion~\citep{candes2010power}. 
Specifically, we deal with the huge number of entries by randomly sample task pairs to evaluate the $\S_{ij}$ and $\S_{ji}$ scores; and deal with the unreliable entries by keeping only task pairs $(i,j)$ with consistent $\S_{ij}$ and $\S_{ji}$ scores.
Given a set of $n$ tasks, we first construct an $n\times n$ partially-observed matrix $\Y$, where its observed entries correspond to the sampled and reliable task pairs $(i,j)$ with consistent $\S_{ij}$ and $\S_{ji}$ scores.
Otherwise, if the task pairs $(i,j)$ are not sampled to compute the transfer scores or the scores are inconsistent, we mark both $\Y_{ij}$ and $\Y_{ji}$ as unobserved.
Given the constructed partially-observed matrix $\Y$, our next step is to recover an $n\times n$ full similarity matrix  using a robust matrix completion approach, and then generate the final task partition by applying spectral clustering to the completed similarity matrix. 
The proposed approach has a 2-fold advantage. First, our method carries a strong theoretical guarantee, showing that the full similarity matrix can be {\it perfectly} recovered if the number of observed correct entries in the partially observed similarity matrix is at least $O(n\log^2n)$. This theoretical result allows us to only compute the similarities of $O(n\log^2n)$ instead of $O(n^2)$ pairs, thus greatly reduces the computation when the number of tasks is large. 
Second, by filtering out uncertain task pairs, the proposed algorithm will be less sensitive to noise, leading to a more robust clustering performance.

The task clusters allow us to handle (i) diverse MTL problems, by model sharing only within clusters such that the negative transfer from irrelevant tasks can be alleviated; and 
(ii) diverse FSL problems, where a new task can be assigned a task-specific metric, which is a linear combination of the metrics defined by different clusters, such that the diverse few-shot tasks could derive different metrics from the previous learning experience.
Our results show that the proposed task clustering algorithm, combined with the above MTL and FSL strategies, could give us 
significantly better deep MTL and FSL algorithms on sentiment classification and intent classification tasks.

\vspace{-0.1 in}
\section{Related Work}
\label{sec:related}
\vspace{-0.1 in}

{\noindent \bf Task/Dataset Clustering on Model Parameters} \quad
This class of task clustering methods measure the task relationships in terms of model parameter similarities on individual tasks. Given the parameters of convex models, task clusters and cluster assignments could be derived via matrix decomposition~\citep{kumar2012learning} or k-means based approach~\citep{kang2011learning}. 
The parameter similarity based task clustering method for deep neural networks~\citep{yang2016deep} applied low-rank tensor decomposition of the model layers from multiple tasks. This method is infeasible for our MATL setting because of its high computation complexity with respect to the number of tasks and its inherent requirement on closely related tasks because of its parameter-similarity based approach.

{\noindent \bf Task/Dataset Clustering with Clustering-Specific Training Objectives} \quad
Another class of task clustering methods joint assign task clusters and train model parameters for each cluster that minimize training loss within each cluster by K-means based approach~\citep{crammer2012learning} or minimize overall training loss combined with sparse or low-ranker regularizers with convex optimization~\citep{barzilai2015convex,murugesan2017co}. Deep neural networks have flexible representation power and they may overfit to arbitrary cluster assignment if we consider training loss alone.
Also, these methods require identical class label sets across different tasks, which does not hold in most of the real-world MATL settings. 

{\noindent \bf Few Shot Learning} \quad
FSL~\citep{li2006one,miller2000learning} aims to learn classifiers for new classes with only a few training examples per class. Bayesian Program Induction~\citep{lake2015human} represents concepts as simple programs that best explain observed examples under a Bayesian criterion.  Siamese neural networks rank similarity between inputs~\citep{koch2015siamese}. Matching Networks~\citep{vinyals2016matching} maps a small labeled support set and an unlabeled example to its
label, obviating the need for fine-tuning to adapt to new class types. These approaches essentially learn one metric for all tasks, which is sub-optimal when the tasks are diverse. An LSTM-based meta-learner~\citep{ravi2017optimization} learns the exact optimization algorithm used to train
another learner neural-network classifier for the few-shot setting. However, it requires uniform classes across tasks. Our FSL approach can handle the challenges of diversity and varying sets of class labels.

\begin{figure}[ht]
\centering
\includegraphics[scale=0.247]{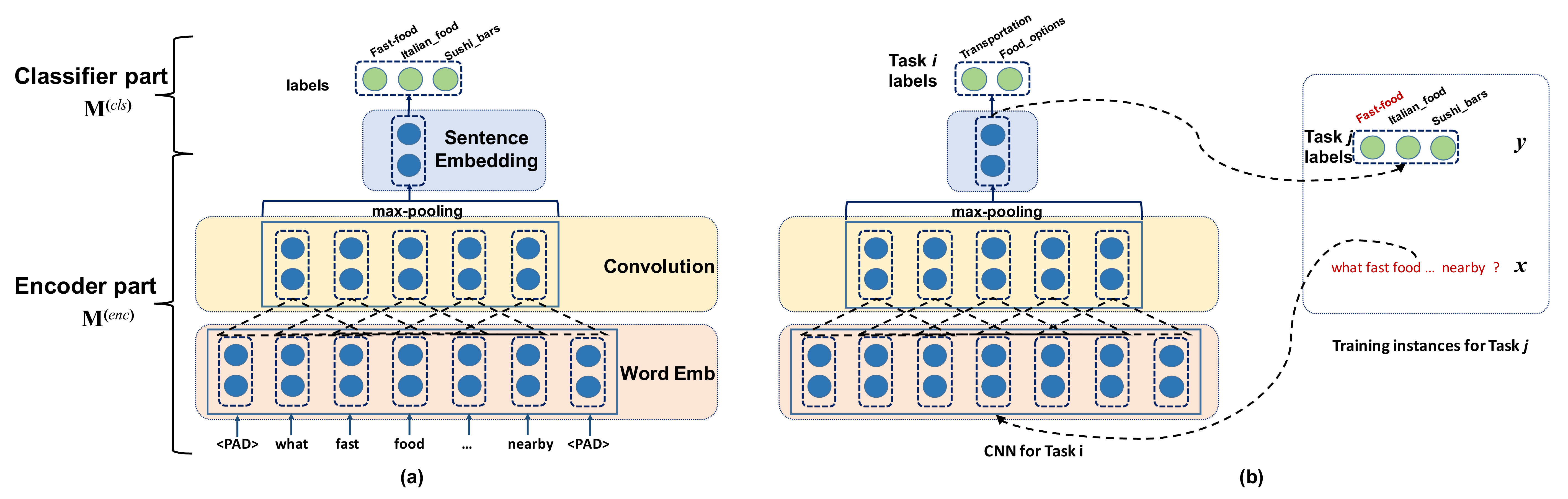}
\caption{{The convolutional neural networks used in this work: (a) a single-task CNN. The encoder component takes the sentence as input and outputs a fixed-length sentence embedding vector; the classifier component predicts class labels with the sentence embedding. (b) the evaluation on transfer performance from task $i$ to $j$, where the encoder of task $i$ was taken to encode task $j$'s sentences and then predict task $j$'s labels (dashed arcs).}
}
\label{fig:basic_models}
\end{figure}

\section{Methodology}
\label{sec:methods}
\label{sec:notations}
Let $ \mathcal{T} = \left \{ \mathrm{T}_1, \mathrm{T}_2, \cdots, \mathrm{T}_n  \right \}$ be the set of $n$ tasks to be clustered, and each task $\mathrm{T}_i$ consists of a train/validation/test data split $\left \{ D^{train}_{i}, D^{valid}_{i}, D^{test}_{i} \right\}$. 
We consider text classification tasks, comprising labeled examples $\{x, y\}$, where the input $x$ is a sentence or document (a sequence of words) and $y$ is the label.
We first train each classification model $\mathrm{M}_i$ on its training set $D^{train}_{i}$, which yields a set of models $\mathcal{M} = \left \{ \mathrm{M}_1,\mathrm{M}_2, \cdots, \mathrm{M}_n \right \}$.
We use convolutional neural network (CNN), which has reported results near state-of-the-art on text classification \citep{kim:2014:EMNLP2014,johnson2016supervised}. 
CNNs also train faster than recurrent neural networks~\citep{hochreiter1997long}, making large-$n$ MATL scenarios more feasible.
Figure \ref{fig:basic_models} shows the CNN architecture. Following \citep{collobert2011natural,kim:2014:EMNLP2014}, the model consists of a convolution layer and a max-pooling operation over the entire sentence. The model has two parts: an \emph{encoder} part and a \emph{classifier} part. Hence each model $\mathrm{M}_i=\{ \mathrm{M}^{enc}_i, \mathrm{M}^{cls}_i \}$.
The above broad definitions encompasses other classification tasks (e.g. image classification) and other classification models (e.g. LSTMs~\citep{hochreiter1997long}).

We propose a task-clustering framework for both multi-task learning (MTL) and few-shot learning (FSL) settings. In this framework, we have the MTL and FSL algorithms summarized in Section \ref{ssec:method_meta_model} \& \ref{ssec:method_fsl}, where our task-clustering framework serves as the initial step in both algorithms.
Figure \ref{fig:basic_idea} gives an overview of our idea and an example on how our task-clustering algorithm helps MTL.

\begin{figure}[ht]
\centering
\includegraphics[scale=0.48]{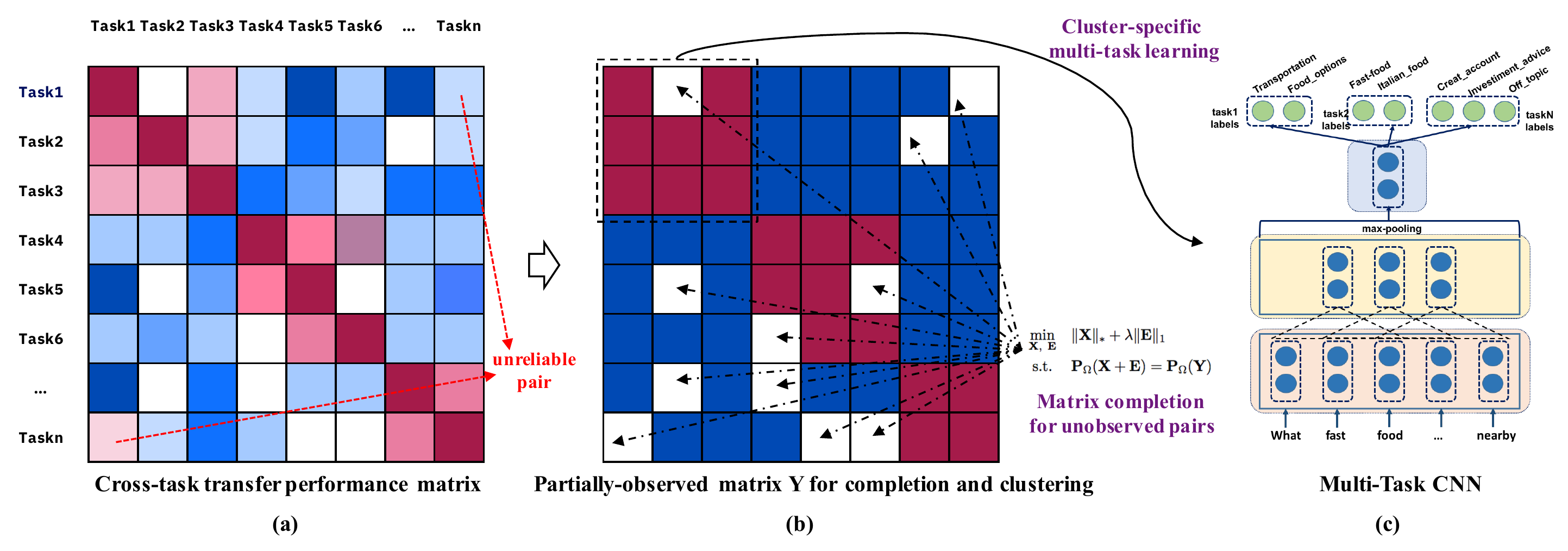}
\caption{{Overview of the idea of our Robust Clustering method with multi-task learning as an example application. (a) an illustration of the sparse cross-tasks transfer-performance matrix with unobserved entries (white blocks) and unreliable values (top-right and bottom-left corners), where red colors indicate positive transfer and blue colors indicate negative transfer; 
(b) the constructed binary partially-observed matrix with low-rank constraint for matrix completion and clustering (see Section \ref{ssec:method_completion} for the detailed mathematics); (c) a multi-task Convolutional Neural Network (MTL-CNN) architecture trained on a task cluster (tasks 1, 2 and 3 in the example).}
}
\label{fig:basic_idea}
\end{figure}

\subsection{Cross-Task Transfer-Performance Matrix Estimation}
\label{ssec:encoder_transfer}
Using single-task models, we can compute performance scores $s_{ij}$ by adapting each $\mathrm{M}_i$ to each task $T_j (j\neq i)$. This forms an 
$n \times n$ pair-wise classification performance matrix $\textbf{S}$,
called the \emph{transfer-performance matrix}.
Note that $\textbf{S}$ is asymmetric since usually $\S_{ij} \neq \S_{ji}$.

When \textbf{all tasks have identical label sets}, we can directly evaluate the model $\mathrm{M}_i$ on the training set of task $j$, $D^{train}_{j}$, and use the accuracy as the cross-task transfer score $\S_{ij}$.

When \textbf{tasks have different label sets}, 
we freeze the encoder $\mathrm{M}^{enc}_i$ from $\rM_i$, on top of which we use $D^{train}_j$ to train a classifier layer. 
This gives us a new task $j$ model, and we test this model on $D^{valid}_j$ to get the accuracy as the transfer-performance $\S_{ij}$.
The score shows how the representations learned on task $i$ can be adapted to task $j$, thus indicating the similarity between tasks.

\textbf{Task Pair Sampling:}\quad 
When the number of tasks $n$ is very large, the evaluation of $O(n^2)$ entries is time-consuming. Thus we sample $n'$ pairs of tasks $\{i,j\}$~$(i\neq j)$, with $n' \ll n$. Then we set $\S_{ij}$ and $\S_{ji}$ as the transfer performance defined above when $\{i,j\}$ is in the $n'$ samples, otherwise the entry is marked as ``\emph{unobserved}''~\footnote{We set all $\S_{ii}=1$.}.

\setlength{\textfloatsep}{0pt}
\begin{algorithm}[htbp]
\small
{
\SetKwInOut{Input}{Input}
\SetKwInOut{Output}{Output}
 \Input{A set of $n$ tasks $ \mathcal{T} = \left \{ \mathrm{T}_1, \mathrm{T}_2, \cdots, \mathrm{T}_n  \right \}$, number of task clusters $K$}
 \Output{$K$ task clusters $C_{1:K}$}
 \DontPrintSemicolon
\BlankLine
 \textbf{Learning of Single-Task Models}: train single-task models $\mathrm{M}_i$ for each task $\mathrm{T}_i$\;
 \textbf{Evaluation of Transfer-Performance Matrix}: get performance matrix $\mathbf{\S}$ (Section \ref{ssec:encoder_transfer})\;
 \textbf{Score Filtering}: Filter the uncertain scores in $\S$ and construct the symmetric matrix $\Y$ using Eq. (\ref{eqn:A})\;
 \textbf{Matrix Completion}: Complete the similar matrix $\X$ from $\Y$ using Eq. (\ref{eqn:pro}) \;
 \textbf{Task Clustering}: $C_{1:K}=SpectralClustering(\X, K)$\;
 \caption{\label{algo:task-clustering}{\robusttc: Robust Task Clustering based on Matrix Completion}}}
\end{algorithm}

\subsection{Robust Task Clustering by Matrix Completion}
\label{ssec:method_completion}
As discussed in the introduction, directly generating the full matrix $\S$ and partitioning tasks based on it has the following disadvantages: (i) there are too many entries to evaluate when the number of tasks is large; (ii) some task pairs are uncertain, thus can mislead the clustering algorithm to output an incorrect task-partition; and (iii) $\S$ is asymmetric, thus cannot be directly analyzed by many conventional clustering methods. We address the first issue by randomly sample some task pairs to evaluate, as described in Section~\ref{ssec:encoder_transfer}. Besides, we address the other issues by constructing a symmetric similarity matrix and only consider the reliable task relationships, as will be introduced in Eq. (\ref{eqn:A}). 
Below, we describe our method (summarized in Algorithm \ref{algo:task-clustering}) in detail.

First, we use only reliable task pairs to generate a {\it partially-observed} similarity matrix $\Y$. Specifically, if $\S_{ij}$ and $\S_{ji}$ are high enough,  
then it is likely that tasks $\{i,j\}$ belong to a same cluster and share significant information. Conversely, if
$\S_{ij}$ and $\S_{ji}$ are low enough, 
then they tend to belong to different clusters. To this end, we need to design a mechanism to determine if a performance is high or low enough. Since different tasks may vary in difficulty, a fixed threshold is not suitable.
Hence, we define a dynamic threshold using the mean and standard deviation of the target task performance, i.e., $\mu_j = \text{mean}(\S_{:j})$ and $\sigma_j=\text{std}(\S_{:j})$, where $\S_{:j}$ is the $j$-th column of
$\S$. We then introduce two positive parameters $p_1$ and $p_2$, and define high and low performance as $\S_{ij}$ greater than $\mu_j + p_1 \sigma_j$ or lower than $\mu_j - p_2 \sigma_j$, respectively. When both $\S_{ij}$ and $\S_{ji}$ are high and low enough, we set their pairwise similarity as $1$ and $0$, respectively. Other task pairs are treated as uncertain task pairs and are marked as unobserved, and will have no influence to our clustering method.
This leads to a partially-observed symmetric matrix $\Y$, i.e., 
\begin{eqnarray}
\Y_{ij}=\Y_{ji}=\left\{
\begin{array}{ll}
1     & \text{if}\ \ \S_{ij} > \mu_j + p_1 \sigma_j\ \ \text{and}\ \ \S_{ji} > \mu_i + p_1 \sigma_i\\
0     & \text{if}\ \ \S_{ij} < \mu_j - p_2 \sigma_j\ \ \text{and}\ \ \S_{ji} < \mu_i - p_2 \sigma_i\\
\mathrm{unobserved} & \mathrm{otherwise}
\end{array}
\right. \label{eqn:A}
\end{eqnarray}

Given the partially observed matrix $\Y$, we then reconstruct the full similarity matrix $\X \in \mathbb{R}^{n\times n}$.
We first note that the similarity matrix $\X$ should be of low-rank (proof deferred to appendix). 
Additionally, since the observed entries of $\Y$ are generated based on high and low enough performance, it is safe to assume that most observed entries are correct and only a few may be incorrect. Therefore, we introduce a sparse matrix $\E$ to capture the observed incorrect entries in $\Y$. Combining the two observations, $\Y$ can be decomposed into the sum of two matrices $\X$ and $\E$, where $\X$ is a low rank matrix storing similarities between task pairs, and $\E$ is a sparse matrix that captures the errors in $\Y$.  The matrix completion problem can be cast as the following convex optimization problem:
\begin{eqnarray}\label{eqn:pro}
&\min\limits_{\X,\ \E} & \|\X\|_* + \lambda \|\E\|_1\\ \label{eqn:B}
& \mbox{s.t.}& \P_{\Omega}(\X+\E)  =  \P_{\Omega}(\Y), \nonumber
\end{eqnarray}
where $\|\circ\|_*$ denotes the matrix nuclear norm, the convex surrogate of rank function. $\Omega$ is the set of observed entries in $\Y$, and $\P_{\Omega}:\R^{n\times n} \mapsto \R^{n\times n}$ is a matrix projection operator defined as
\begin{eqnarray}
[\P_{\Omega}(\A)]_{ij} = \left\{
\begin{array}{ll}
\A_{ij} & \text{if}\ (i,j) \in \Omega \nonumber\\
0 & \mbox{otherwise}\nonumber
\end{array}
\right. \label{eqn:p}
\end{eqnarray}
The following theorem shows the perfect recovery guarantee for the problem (\ref{eqn:pro}). The proof is deferred to Appendix.

\begin{theorem}\label{thm:perfect-recovery}
Let $\X^* \in \R^{n\times n}$ be a rank $k$ matrix with a singular value decomposition  $\X^* = \U\Sigma \V^{\top}$, where $\U = (\u_1, \ldots, \u_k) \in \R^{n\times k}$ and $\V = (\v_1, \ldots, \v_k) \in \R^{n\times k}$ are the left and right singular vectors of $\X^*$, respectively. Similar to many related works of matrix completion, we assume that the following two assumptions are satisfied:
\begin{enumerate}
\item The row and column spaces of $\X$ have coherence bounded above by a positive number $\mu_0$.
\item Max absolute value in matrix\ $\U\V^{\top}$ is bounded above by $\mu_1\sqrt{r}/n$ for a positive number $\mu_1$.
\end{enumerate}
Suppose that $m_1$ entries of $\X^*$ are observed with their locations sampled uniformly at random, and among the $m_1$ observed entries, $m_2$ randomly sampled entries are corrupted. Using the resulting partially observed matrix as the input to the problem (\ref{eqn:pro}), then with a probability at least $1 - n^{-3}$, the underlying matrix $\X^*$ can be perfectly recovered, given 
\begin{enumerate}
\item $\mu(\E)\xi(\X) \leq \frac{1}{4k + 5}$,
\item $\frac{\xi(\X) - (2k -1)\mu(\E)\xi^2(\X)}{1 - 2(k+1)\mu(\E)\xi(\X)} < \lambda < \frac{1 - (4k+5)\mu(\E)\xi(\X)}{(k+2)\mu(\E)}$,
\item $\ m_1 - m_2 \geq C[\max(\mu_0, \mu_1)]^4n\log^2 n$,
\end{enumerate}
where $C$ is a positive constant; $\xi(\circ)$ and $\mu(\circ)$ denotes the low-rank and sparsity incoherence~\citep{chandrasekaran2011rank}.
\end{theorem}
Theorem~\ref{thm:perfect-recovery} implies that even if some of the observed entries computed by (\ref{eqn:A}) are incorrect, problem (\ref{eqn:pro}) can still perfectly recover the underlying similarity matrix $\X^*$ if the number of observed correct entries is at least $O(n \log^2 n)$. 
For MATL with large $n$, this implies that only a tiny fraction of all task pairs is needed to reliably infer similarities over all task pairs.
Moreover, the completed similarity matrix $\X$ is symmetric, due to symmetry of the  input matrix $\Y$.
This enables analysis by similarity-based clustering algorithms, such as spectral clustering.

\subsection{Multi-Task Learning Based on Tasks Clusters}
\label{ssec:method_meta_model}

\begin{algorithm}[h!]
\small
{
\SetKwInOut{Input}{Input}
\SetKwInOut{Output}{Output}
 \Input{A set of $n$ tasks $ \mathcal{T} = \left \{ \mathrm{T}_1, \mathrm{T}_2, \cdots, \mathrm{T}_n  \right \}$; number of clusters $K$}
 \Output{$K$ task clusters $C_{1:K}$ and cluster-models $\mathbf{\Lambda} = \left \{ \Lambda_1,\Lambda_2, \cdots, \Lambda_K \right \}$}
 \DontPrintSemicolon
\BlankLine
 \textbf{Robust Task Clustering}: $C_{1:K}$ = \robusttc($\mathcal{T}$,$K$) (Algorithm \ref{algo:task-clustering}) \;
 \textbf{Cluster-Model Training}: Train one multi-task model $\mathrm{\Lambda}_i$ on each task cluster $C_i$ (Section \ref{ssec:method_meta_model})\;
 \caption{\label{algo:taskcluster-mtl}{\robusttc-MTL: Multi-Task Learning based on Task Clustering}}
}
\end{algorithm}

For each cluster $C_k$, we train a model $\Lambda_k$ with all tasks in that cluster to encourage parameter sharing. We call $\Lambda_k$ the \textbf{cluster-model}.
When evaluated on the MTL setting,
with sufficient data to train a task-specific classifier, we only share the encoder part and have distinct task-specific classifiers (Figure~\ref{fig:basic_models}(b)). These task-specific classifiers provide flexibility to handle varying number of labels. 

\subsection{Few-Shot Learning Based on Tasks Clusters}
\label{ssec:method_fsl}

\begin{algorithm}[h]
\small
{
\SetKwInOut{Input}{Input}
\SetKwInOut{Output}{Output}

 \Input{$N$ training tasks $ \mathcal{T} = \left \{ \mathrm{T}_1, \mathrm{T}_2, \cdots, \mathrm{T}_n  \right \}$; number of clusters $K$; target few-shot learning task $\mathrm{T}_{trg}$}
 \Output{A classification model for the target task $\mathrm{M}_{trg}$, $K$ task clusters $C_{1:K}$ and cluster-models $\mathbf{\Lambda} = \left \{ \Lambda_1,\Lambda_2, \cdots, \Lambda_K \right \}$}
 \DontPrintSemicolon
\BlankLine
 \textbf{Learning Cluster-Models on Training Tasks}: $C_{1:K}$, $\mathbf{\Lambda}$ = \robusttc-MTL($\mathcal{T}$,$K$) (Algorithm \ref{algo:taskcluster-mtl}) \;
 \textbf{Few-Shot Learning on Cluster-models}: Train a model $\mathrm{M}_{trg}$ on task $\mathrm{T}_{trg}$ with the method in Section \ref{ssec:method_fsl}.
 \caption{\label{algo:taskcluster-fsl}{\robusttc-FSL: Task Clustering for Few-Shot Learning}}}
\end{algorithm}

We only have access to a limited number of training samples in few-shot learning setting, so it is impractical to train well-performing task-specific classifiers as in the multi-task learning setting. Instead, we make the prediction of a new task by linearly combining prediction from learned clusters.
\begin{align}
p(y|x) = \sum_k \alpha_k P(y|x; \Lambda_k).
\label{eqn:fsl}
\end{align}
where $\Lambda_k$ is the learned (and frozen) model of the $k$-th cluster, $\{\alpha_{k}\}_{k=1}^{K}$ are adaptable parameters. 

We use some alternatives to train cluster-models $\Lambda_k$, which could better suit (and is more consistent to) the above FSL method.\footnote{We also tried these alternatives under the MTL settings, which perform worse than MTL-CNN.}
When \textbf{all tasks have identical label sets},
we train a single classification model on all the tasks like in previous work \citep{barzilai2015convex}, the predictor $P(y|x; \Lambda_k)$ is directly derived from this cluster-model. 
When \textbf{tasks have different label sets },
we train a metric-learning model like \citep{vinyals2016matching} among all the tasks in $C_k$, which consist a shared encoding function $\Lambda^{enc}_k$ aiming to make each example closer to examples with the same label compared to the ones with different labels.

Then we use the encoding function to derive the predictor by 
\begin{align}
P(y=y_l|x;\Lambda_k) = \frac{\exp\left \{ \Lambda^{enc}_k (x_l)^{\top}\Lambda^{enc}_k (x) \right \} }{\sum_{l'} \exp \left \{ \Lambda^{enc}_k (x_{l'})^{\top}\Lambda^{enc}_k (x) \right \} }
\end{align}
where $x_{l}$ is the corresponding training sample for label $y_{l}$. 

\section{Experiments}

\subsection{Experiment Setup}
\label{ssec:exp_setup}

{\noindent \bf Data Sets}\ \ \ 
We test our methods by conducting experiments on three text classification data sets. In the data-preprocessing step we used NLTK toolkit\footnote{\url{http://www.nltk.org/}} for tokenization. For MTL setting, all tasks are used for clustering and model training. For FSL setting, the task are divided into training tasks and testing tasks (target tasks), where the training tasks are used for clustering and model training, the testing tasks are few-shot learning ones used to for evaluating the method in Eq. (\ref{eqn:fsl}).

\begin{enumerate}
\item {\noindent \bf Amazon Review Sentiment Classification}\ \ \ 
First, following \citet{barzilai2015convex}, we construct a multi-task learning setting with the multi-domain sentiment classification~\citep{blitzer2007biographies} data set. 
The dataset consists of Amazon product reviews for 23 types of products (see Appendix 3 for the details).
For each domain, we construct three binary classification tasks with different thresholds on the ratings: the tasks consider a review as positive if it belongs to one of the following buckets =5 stars, >=4 stars or >=2 stars \footnote{Data downloaded from \url{http://www.cs.jhu.edu/~mdredze/datasets/sentiment/}, in which the 3-star samples were unavailable due to their ambiguous nature \citep{blitzer2007biographies}.}
These review-buckets then form the basis of the task-setup for MATL, giving us $23 \times 3=69$ tasks in total. For each domain we distribute the reviews uniformly to the three tasks.
For evaluation, we select tasks from 4 domains ({\it Books, DVD, Electronics, Kitchen}) as the target tasks (12 tasks) out of all 23 domains. For FSL evaluation, we create five-shot learning tasks on the selected target tasks. The cluster-models for this evaluation are standard CNNs shown in Figure \ref{fig:basic_models}(a), and we share the same output layer to evaluate the probability in Eq. (\ref{eqn:fsl}) as all tasks have the same number of labels.
\item {\noindent \bf Diverse Real-World Tasks: User Intent Classification for Dialog System}\ \ \ 
The second dataset is from an on-line service which trains and serves intent classification models to various clients. The dataset comprises recorded conversations between human users and dialog systems in various domains, ranging from personal assistant to complex service-ordering or a customer-service request scenarios. During classification, intent-labels\footnote{In conversational dialog systems, intent-labels are used to guide the dialog-flow.} are assigned to user utterances (usually sentences). We use a total of 175 tasks from different clients, and randomly sample 10 tasks from them as our target tasks. For each task, we randomly sample 64\% data into a training set, 16\% into a validation set, and use the rest as the test set (see Appendix 3 for details).
The number of labels for these tasks vary from 2 to 100. Hence, to adapt this to a FSL scenario, we keep one example for each label (one-shot), plus 20 randomly picked labeled examples to create our training data. We believe this is a fairly realistic estimate of labeled examples one client could provide easily. Since we deal with various number of labels in the FSL setting, we chose matching networks \citep{vinyals2016matching} as the cluster-models. 
\item {\noindent \bf Extra-Large Number of Real-World Tasks}\ \ \ 
Similar to the second dataset, we further collect 1,491 intent classification tasks from the on-line service.
This setting is mainly used to verify the robustness of our task clustering method, since it is difficult to estimate the full transfer-performance matrix $\S$ in this setting (1,491$^2$=2.2M entries).
Therefore, in order to extract task clusters, we randomly sample task pairs from the data set to obtain 100,000 entries in $\S$, which means that only about $100\textrm{K}/2.2\textrm{M} \approx 4.5\%$ of the entries in $\S$ are observed. The number of 100,000 is chosen to be close to $n\log^2 n$ in our theoretical bound in Theorem \ref{thm:perfect-recovery}, so that we could also verify the tightness of the bound empirically.
To make the best use of the sampled pairs, in this setting we modified the Eq. \ref{eqn:A}, so that each entry $\Y_{ij}=\Y_{ji}=1$ if $\S_{ij} \geq \mu_j$ or $\S_{ji} \geq \mu_i$ and $\Y_{ij}=0$ otherwise. In this way we could have determined number of entries in $\Y$ as well, since all the sampled pairs will correspond to observed (but noisy) entries in $\Y$. We only run MTL setting on this data set.
\end{enumerate}
{\noindent \bf Baselines}\ \ 
\ For MTL setting, we compare our method to the following baselines: (1) \textbf{single-task CNN}: training a CNN model for each task individually; (2) \textbf{holistic MTL-CNN}: training one MTL-CNN model (Figure \ref{fig:basic_models}b) on all tasks; (3) \textbf{holistic MTL-CNN (target only)}: training one MTL-CNN model on all the target tasks.
For FSL setting, the baselines consist of: (1) \textbf{single-task CNN}: training a CNN model for each task individually; (2) \textbf{single-task FastText}: training one FastText model~\citep{joulin2016bag} with fixed embeddings for each individual task; 
(3) \textbf{Fine-tuned the holistic MTL-CNN}: fine-tuning the classifier layer on each target task after training initial MTL-CNN model on all training tasks; (4) \textbf{Matching Network}: a metric-learning based few-shot learning model trained on all training tasks. We initialize all models with pre-trained 100-dim Glove embeddings (trained on 6B corpus)~\citep{pennington2014glove}.

As the intent classification tasks usually have various numbers of labels, to our best knowledge the proposed method is the only one supporting task clustering in this setting; hence we only compare with the above baselines.
Since sentiment classification involves binary labels, we compare our method with the 
state-of-the-art logistic regression based task clustering method (\textbf{ASAP-MT-LR})~\citep{barzilai2015convex}. We also try another approach where we run our MTL/FSL methods on top of the (\textbf{ASAP-Clus-MTL/FSL}) clusters (as their entire formulation is only applicable to convex models).

{\noindent \bf Hyper-Parameter Tuning}\ \ 
In all experiments, we set both $p_1$ and $p_2$ parameters in (\ref{eqn:A}) to $0.5$.
This strikes a balance between obtaining enough observed entries in $\Y$, and ensuring that most of the retained similarities are consistent with the cluster membership.
For MTL settings, we tune parameters like the window size and hidden layer size of CNN, learning rate and the initialization of embeddings (random or pre-trained) based on average accuracy on the union of all tasks' dev sets, in order to find the best identical setting for all tasks. Finally we have the CNN with window size of 5 and 200 hidden units. The learning rate is selected as 0.001; and all MTL models use random initialized word embeddings on sentiment classification and use Glove embeddings as initialization on intent classification, which is likely because the training sets of the intent tasks are usually small. We also used the early stopping criterion based on the previous condition.

For the FSL setting, hyper-parameter selection is difficult since there is no validation data (which is a necessary condition to qualify as a $k$-shot learning). So, in this case we preselect a subset of training tasks as validation tasks and tune the learning rate and training epochs (for the rest we follow the best setting from the MTL experiments) on the validation tasks. During the testing phase (i.e. model training on the target FSL tasks), we fix the selected hyper-parameter values for all the algorithms.

{\noindent \bf Out-of-Vocabulary in Transfer-Performance Evaluation}\ \ \ 
In text classification tasks, transferring an encoder with fine-tuned word embeddings from one task to another may not work as there can be a significant difference between the vocabularies. Hence, while learning the single-task models (line 1 of Algorithm \ref{algo:task-clustering}) we always use the CNNs with fixed set of pre-trained embeddings.

\subsection{Sentiment Classification on Amazon Product Reviews}
\label{ssec:exp_sentiment}
{\noindent \bf Improving Observed Tasks (MTL Setting)}\ \ \ 
\begin{table}[ht]
\caption{ Accuracy on the 12 target sentiment classification tasks.}
\label{tab:sa_exp}
\begin{subtable}{0.45\textwidth}
\small
\centering
\caption{\label{tab:sa_mtl}MTL setting (i.e. training on all 69 tasks).}
\begin{tabular}{l|c}
\hline 
\bf Model  & \bf Avg Acc \\ \hline
Single-task CNN   & 85.51  \\ 
ASAP-MTLR~\tiny{\citep{barzilai2015convex}} & 85.17 \\
Holistic MTL-CNN  & 85.23  \\
Holistic MTL-CNN (target only)  & 85.71 \\
\hline
\hline
\multicolumn{2}{l}{\bf \robusttc-MTL} \\
\hline
\quad clus=5 & 86.13 \\
\quad clus=10 & \bf 86.73 \\
\hline
\multicolumn{2}{l}{\bf ASAP-Clus-MTL}\\
\hline
\quad clus=5 & 86.07 \\
\quad clus=10 & 85.60 \\
\hline
\end{tabular}

\end{subtable}%
~~~~~~~~~~
\begin{subtable}{0.45\textwidth}
\small
\centering
\caption{\label{tab:sa_fsl}Few-shot learning setting (five-shot).}
\begin{tabular}{l|c}
\hline 
\bf Model  & \bf Avg Acc \\ \hline
Single-task CNN w/ pre-trained emb   & 65.92  \\ 
Single-task FastText w/ pre-trained emb   & 63.05 \\ 
Fine-tuned the holistic MTL-CNN   & 76.56 \\
Matching Network~\citep{vinyals2016matching} & 65.73  \\
\hline
\hline
\multicolumn{2}{l}{\bf \robusttc-FSL} \\
\hline
\quad clus=5 & \bf  83.12 \\
\quad clus=10 &  81.96 \\
\quad clus=num\_of\_tasks (no clustering) & 78.85\\
\hline
\multicolumn{2}{l}{\bf ASAP-Clus-FSL} \\
\hline
\quad clus=5 & 82.65 \\
\quad clus=10 & 81.44 \\
\hline
\end{tabular}
\label{tab:table1_b}
\end{subtable}%
\end{table}
Table \ref{tab:sa_exp} shows the results of the 12 target tasks when all 69 tasks are used for training. Since most of the tasks have a significant amount of training data, the single-task baselines achieve good results. Because the conflicts among some tasks (e.g. the 2-star bucket tasks and 5-star bucket tasks require opposite labels on 4-star examples), the holistic MTL-CNN does not show accuracy improvements compared to the single-task methods. It also lags behind the holistic MTL-CNN model trained only on 12 target domains, which indicates that the holistic MTL-CNN cannot leverage large number of background tasks. Our \robusttc-MTL method based on task clustering achieves a significant improvement over all the baselines.

The ASAP-MTLR (best score achieved with five clusters) could improve single-task linear models with similar merit of our method. However, it is restricted by the representative strength of linear models so the overall result is lower than the deep learning baselines. 

{\noindent \bf Adaptation to New Tasks (FSL Setting)}\ \ \ 
Table \ref{tab:sa_exp}(b) shows the results on the 12 five-shot tasks by leveraging the learned knowledge from the 57 previously observed tasks. Due to the limited training resources, all the baselines perform poorly. 
Our \robusttc-FSL gives far better results compared to all baselines ($>$6\%). It is also significantly better than applying Eq. (\ref{eqn:fsl}) without clustering (78.85\%), i.e. using single-task model from each task instead of cluster-models for $P(y|x;\cdot)$. 

{\noindent \bf Comparison to the ASAP Clusters}\ \ 
Our clustering-based MTL and FSL approaches also work for the ASAP clusters, in which we replace our task clusters with the task clusters generated by ASAP-MTLR. In this setting we get a slightly lower performance compared to the \robusttc-based ones on both MTL and FSL settings, but overall it performs better than the baseline models.
This result shows that, apart from the ability to handle varying number of class labels, our \robusttc\ model can also generate better clusters for MTL/FSL of deep networks, even under the setting where all tasks have the same number of labels.

It is worth to note that from Table~\ref{tab:sa_exp}(a), training CNNs on the ASAP clusters gives better results compared to training logistic regression models on the same 5 clusters (86.07 vs. 85.17), despite that the clusters are not optimized for CNNs. Such result further emphasizes the importance of task clustering for deep models, when better performance could be achieved with such models. 

\subsection{User Intent Classification from Diverse Real-World Online Services}
\label{ssec:exp_large}

\begin{table}[ht]
\caption{ Accuracy on the 10 dialog intent classification tasks.}
\label{tab:nlu_exp}
\begin{subtable}{0.45\textwidth}
\small
\centering
\caption{\label{tab:nlu_mtl}MTL setting (i.e. training on all 175 tasks).}
\begin{tabular}{l|c}
\hline 
\bf Model  & \bf Avg Acc \\ \hline
Single-task CNN   & 58.47  \\ 
Holistic MTL-CNN & 62.42  \\
Holistic MTL-CNN (target only) & 62.45 \\
\hline
\hline
\multicolumn{2}{l}{\bf \robusttc-MTL} \\
\hline
\quad clus=10 & 64.41 \\
\quad clus=20 & \bf 68.11 \\
\quad clus=30 & 66.74 \\
\hline
\end{tabular}

\end{subtable}%
~~~~~~~~~~
\begin{subtable}{0.45\textwidth}
\small
\centering
\caption{\label{tab:nlu_fsl}FSL setting (one-shot + 20 examples).}
\begin{tabular}{l|c}
\hline 
\bf Model  & \bf Avg Acc \\ \hline
Single-task CNN w/pre-trained emb   & 34.46 \\ 
Single-task FastText w/pre-trained emb   & 23.87 \\ 
Fine-tuned holistic MTL-CNN  & 30.36 \\
Matching Network~\citep{vinyals2016matching} & 30.42  \\
\hline
\hline
\multicolumn{2}{l}{\bf \robusttc-FSL} \\
\hline
\quad clus=10 & 34.64\\
\quad clus=20 & \bf 37.59\\
\quad clus=30 & 36.82\\
\quad clus=num\_of\_tasks (no clustering) & 34.43 \\
\hline
Adaptive \robusttc-FSL (clus=20) & {\bf 42.97}\\
\hline
\end{tabular}
\end{subtable}%
\end{table}

Table \ref{tab:nlu_exp}(a) \& (b) show the MTL \& FSL results on dialog intent classification, which demonstrates trends similar to the sentiment classification tasks. Note that the holistic MTL methods achieve much better results compared to single-task CNNs. This is because the tasks usually have smaller training and development sets, and both the model parameters learned on training set and the hyper-parameters selected on development set can easily lead to over-fitting. \robusttc-MTL achieves large improvement (5.5\%) over the best MTL baseline, because the tasks here are more diverse than the sentiment classification tasks and task-clustering greatly reduces conflicts from irrelevant tasks.

Although our \robusttc-FSL improves over baselines under the FSL setting, the margin is smaller. This is because of the huge diversity among tasks -- by looking at the training accuracy, we found several tasks failed because none of the clusters could provide a metric that suits the training examples. To deal with this problem, we hope that the algorithm can automatically decide whether the new task belongs to any of the task-clusters. If the task doesn't belong to any of the clusters, it would not benefit from any previous knowledge, so it should fall back to single-task CNN. The new task is treated as ``out-of-cluster'' when none of the clusters could achieve higher than 20\% accuracy (selected on dev tasks) on its training data. We call this method \textbf{Adaptive \robusttc-FSL}, and it gives more than 5\% performance boost over the best \robusttc-FSL result.  

{\noindent \bf Discussion on Clustering-Based FSL}\ \ \ 
The single metric based FSL method (Matching Network) achieved success on homogeneous few-shot tasks like Omniglot and miniImageNet \citep{vinyals2016matching} but performs poorly in both of our experiments. This indicates that it is important to maintain multiple metrics for few-shot learning problems with more diverse tasks, similar to the few-shot NLP problems investigated in this paper.
Our clustering-based FSL approach maintains diverse metrics while keeping the model simple with only $K$ parameters to estimate. 
It is worthwhile to study how and why the NLP problems make few-shot learning more difficult/heterogeneous; and how well our method can generalize to non-NLP problems like miniImageNet. We will leave these topics for future work.

\subsection{Large-Scale User Intent Classification with Task-Pair Sampling}\label{ssec:exp_xl}
Table \ref{tab:nlu1491_exp} shows the MTL results on the extra-large dialog intent classification dataset. Compared to the results on the 175 tasks, the holistic MTL-CNN achieves larger improvement (6\%) over the single-task CNNs, which is a stronger baseline. Similar as the observation on the 175 tasks, here the main reason for its improvement is the consistent development and test performance due to holistic multi-task training approach: both the single-task and holistic multi-task model achieve around 66\% average accuracy on development sets. 
Unlike the experiments in Section~\ref{ssec:exp_large}, we did not evaluate the full transfer-performance matrix $\S$ due to time considerations. Instead, we only use the information of $\sim 4.5\%$ of all the task-pairs, and our algorithm still achieves a significant improvement over the baselines. Note that this result is obtained by only sampling about $n \log^2 n$ task pairs, it not only confirms the empirical advantage of our multi-task learning algorithm, but also verifies the correctness of our theoretical bound in Theorem~\ref{thm:perfect-recovery}.

\begin{table}[ht]
\caption{Accuracy of Multi-Task Learning on the 1,491 dialog intent classification tasks.}
\label{tab:nlu1491_exp}
\centering
\begin{tabular}{l|c}
\hline 
\bf Model  & \bf Avg Acc \\ \hline
Single-task CNN   & 60.49  \\ 
Holistic MTL-CNN & 66.42  \\
\hline
\multicolumn{2}{l}{\bf \robusttc-MTL} \\
\hline
\quad clus=30 & 69.62 \\
\quad clus=40 & \bf 70.50 \\
\quad clus=50 & 69.87 \\
\hline
\end{tabular}
\end{table}

\section{Conclusion}
In this paper, we propose a robust task-clustering method that not only has strong theoretical guarantees but also demonstrates significantly empirical improvements when equipped by our MTL and FSL algorithms. Our empirical studies verify that (i) the proposed task clustering approach is very effective in the many-task learning setting especially when tasks are diverse; (ii) our approach could efficiently handle large number of tasks as suggested by our theory; and (iii) cross-task transfer performance can serve as a powerful task similarity measure. 
Our work opens up many future research directions, such as supporting online many-task learning with incremental computation on task similarities, and combining our clustering approach with the recent learning-to-learn methods (e.g. \citep{ravi2017optimization}), 
to enhance our MTL and FSL methods. 


\bibliography{robusttc}
\bibliographystyle{iclr2018_conference}

\section*{Appendix A: Proof of Low-rankness of Matrix $\X$}
We first prove that the full similarity matrix $\X \in \mathbb{R}^{n\times n}$ is of low-rank. To see this, let  $\A = (\a_1, \ldots, \a_k)$ be the underlying perfect clustering result, where $k$ is the number of clusters and $\a_i \in \{0, 1\}^n$ is the membership vector for the $i$-th cluster. Given $\A$, the similarity matrix $\X$ is computed as
\[
    \X = \sum_{i=1}^k \a_i \a_i^{\top} = \sum_{i=1}^k \B_i
\]
where $\B_i = \a_i \a_i^{\top}$ is a rank one matrix. Using the fact that $\mbox{rank}(\X) \leq \sum_{i=1}^k \mbox{rank}(\B_i)$ and $\mbox{rank}(\B_i) =1$, we have $\mbox{rank}(\X) \leq k$, i.e., the rank of the similarity matrix $\X$ is upper bounded by the number of clusters. Since the number of clusters is usually small,
the similarity matrix $\X$ should be of low rank. 

\section*{Appendix B: Proof of Theorem 4.1}
We then prove our main theorem. First, we define several notations that are used throughout the proof. Let $\X = \U\Sigma \V^{\top}$ be the singular value decomposition of matrix $\X$, where $\U = (\u_1, \ldots, \u_k) \in \R^{n\times k}$ and $\V = (\v_1, \ldots, \v_k) \in \R^{n\times k}$ are the left and right singular vectors of matrix $\X$, respectively. Similar to many related works of matrix completion, we assume that the following two assumptions are satisfied:
\begin{enumerate}
\item {\bf A1}: the row and column spaces of $\X$ have coherence bounded above by a positive number $\mu_0$, i.e., $\sqrt{n/r} \max_{i}\|\P_{\U}(\e_i)\| \leq \mu_0$ and $\sqrt{n/r} \max_{i}\|\P_{\V}(\e_i)\| \leq \mu_0$, where $\P_{\U} = \U\U^{\top}$, $\P_{\V} = \V\V^{\top}$, and $\e_i$ is the standard basis vector, and
\item {\bf A2}: the matrix $\U\V^{\top}$ has a maximum entry bounded by $\mu_1\sqrt{r}/n$ in absolute value for a positive number $\mu_1$.
\end{enumerate}
Let $T$ be the space spanned by the elements of the form $\u_i\y^{\top}$ and $\x\v^{\top}_i$, for $1 \leq i \leq k$, where $\x$ and $\y$ are arbitrary $n$-dimensional vectors. Let $T^{\perp}$ be the orthogonal complement to the space $T$, and let $\P_T$ be the orthogonal projection onto the subspace $T$ given by
\[
    \P_T(\Z) = \P_{\U}\Z + \Z\P_{\V} - \P_{\U}\Z\P_{\V}.
\]

The following proposition shows that for any matrix $\Z \in T$, it is a zero matrix if enough amount of its entries are zero.
\begin{prop}
Let $\Omega$ be a set of $m$ entries sampled uniformly at random from $[1,\ldots, n]\times[1,\ldots, n]$, and $\P_{\Omega}(\Z)$ projects matrix $\Z$ onto the subset $\Omega$. If $m > m_0$, where $m_0 = C_R^2\mu_0rn\beta\log n$ with $\beta > 1$ and $C_R$ being a positive constant, then for any $\Z \in T$ with $\P_{\Omega}(\Z) = 0$, we have $\Z = 0$ with probability $1 - 3n^{-\beta}$.
\end{prop}
\begin{proof}
According to the Theorem 3.2 in \cite{candes2010power}, for any $\Z \in T$, with a probability at least $1 - 2n^{2 - 2\beta}$, we have
\begin{eqnarray}\label{eqn:4}
 \|\P_T(\Z)\|_F - \delta \|\Z\|_F \leq \frac{n^2}{m}\|\P_T\P_{\Omega}\P_T(\Z)\|_F^2 = 0
\end{eqnarray}
where $\delta = m_0/m < 1$. Since $\Z \in T$, we have $P_T(\Z) = \Z$. Then from (\ref{eqn:4}), we have $\|\Z\|_F \leq 0 $ and thus $\Z = 0$.
\end{proof}

In the following, we will develop a theorem for the dual certificate that guarantees the unique optimal solution to the following optimization problem
\begin{eqnarray}\label{eqn:pro_appendix}
&\min\limits_{\X,\ \E} & \|\X\|_* + \lambda \|\E\|_1\\ \label{eqn:B}
& \mbox{s.t.}& \P_{\Omega}(\X+\E)  =  \P_{\Omega}(\Y). \nonumber
\end{eqnarray}
\begin{thm}\label{thm:1}
Suppose we observe $m_1$ entries of $\X$ with locations sampled uniformly at random, denoted by $\Omega$. We further assume that $m_2$ entries randomly sampled from $m_1$ observed entries are corrupted, denoted by $\Delta$. Suppose that $\P_{\Omega}(\Y) = \P_{\Omega}(\X + \E)$ and the number of observed correct entries $m_1 - m_2 > m_0=C_R^2\mu_0rn\beta\log n$. Then, for any $\beta > 1$, with a probability at least $1 - 3n^{-\beta}$, the underlying true matrices $(\X, \E)$ is the unique optimizer of (\ref{eqn:pro_appendix}) if both assumptions {\bf A1} and {\bf A2} are satisfied and there exists a dual $\Q \in \R^{n\times n}$ such that (a) $\Q = \P_{\Omega}(\Q)$, (b) $\P_T(\Q) = \U\V^{\top}$, (c) $\|\P_{T^{\top}}(\Q)\| < 1$, (d) $\P_{\Delta}(\Q) = \lambda\ \sgn(\E)$, and (e) $\|\P_{\Delta^c}(\Q)\|_{\infty} < \lambda$.
\end{thm}
\begin{proof}
First, the existence of $\Q$ satisfying the conditions (a) to (e) ensures that $(\X, \E)$ is an optimal solution. We only need to show its uniqueness and we prove it by contradiction. Assume there exists another optimal solution $(\X+\N_\X, \E+\N_\E)$, where $\P_{\Omega}(\N_\X + \N_\E) = 0$. Then we have
\begin{eqnarray*}
\|\X+\N_\X\|_* + \lambda \|\E+\N_\E\|_1 & \geq & \|\X\|_* +\lambda \|\E\|_1 +  \langle \Q_\E, \N_\E \rangle + \langle \Q_\X, \N_\X \rangle
\end{eqnarray*}
where $\Q_\E$ and $\Q_\X$ satisfying $\P_{\Delta}(\Q_\E) = \lambda\ \sgn(\E)$, $\|\P_{\Delta^c}(\Q_\E)\|_{\infty} \leq \lambda$, $\P_T(\Q_\X) = \U\V^{\top}$ and $\|\P_{T^{\perp}}(\Q_\X)\| \leq 1$. As a result, we have
\begin{eqnarray*}
& & \lambda \|\E+\N_\E\|_1 + \|\X+\N_\X\|_* \\
& \geq & \lambda \|\E\|_1 + \|\X\|_* + \langle \Q + \P_{\Delta^c}(\Q_\E) - \P_{\Delta^c}(\Q), \N_\E \rangle + \langle \Q + \P_{T^{\perp}}(\Q_\X) - \P_{T^{\perp}}(\Q), \N_\X \rangle \\
& = & \lambda \|\E\|_1 + \|\X\|_*  + \langle \Q, \N_\E + \N_\X \rangle + \langle \P_{\Delta^c}(\Q_\E) - \P_{\Delta^c}(\Q), \N_\E \rangle + \langle \P_{T^{\perp}}(\Q_\X) - \P_{T^{\perp}}(\Q), \N_\X \rangle \\
& = & \lambda \|\E\|_1 + \|\X\|_*  + \langle \P_{\Delta^c}(\Q_\E) - \P_{\Delta^c}(\Q), \P_{\Delta^c}(\N_\E) \rangle + \langle \P_{T^{\perp}}(\Q_\X) - \P_{T^{\perp}}(\Q), \P_{T^{\perp}}(\N_\X) \rangle
\end{eqnarray*}

We then choose $\P_{\Delta^c}(\Q_\E)$ and $\P_{T^{\perp}}(\Q_\X)$ to be such that $\langle \P_{\Delta^c}(\Q_\E), \P_{\Delta^c}(\N_\E) \rangle = \lambda \|\P_{\Delta^c}(\N_\E)\|_1$ and $\langle \P_{T^{\perp}}(\Q_\X), \P_{T^{\perp}}(\N_\X) \rangle = \|\P_{T^{\perp}}(\N_\X)\|_{*}$. We thus have
\begin{eqnarray*}
& & \lambda \|\E+\N_\E\|_1 + \|\X+\N_\X\|_*\\
& \geq &\lambda \|\E\|_1 + \|\X\|_*  + (\lambda - \|\P_{\Delta^c}(\Q)\|_{\infty}) \|\P_{\Delta^c}(\N_\E)\|_1 + (1 - \|\P_{T^{\perp}}(\Q)\|)\|\P_{T^{\perp}}(\N_\X)\|_{*}
\end{eqnarray*}
Since $(\X+\N_\X, \E+\N_\E)$ is also an optimal solution, we have $\|\P_{\Omega^c}(\N_E)\|_1 = \|\P_{T^{\perp}}(\N_\X)\|_{*}$, leading to $\P_{\Omega^c}(\N_\E) = \P_{T^{\perp}}(\N_\X) = 0$, or $\N_\X \in T$. Since $\P_{\Omega}(\N_\X + \N_\E) = 0$, we have $\N_\X = \N_\E + \Z$, where $P_{\Omega}(\Z) = 0$ and $\P_{\Omega^c}(\N_\E) = 0$. Hence, $\P_{\Omega^c \cap \Omega}(\N_\X) = 0$, where $|\Omega^c \cap \Omega| = m_1 - m_2$. Since $m_1 - m_2 > m_0$, according to Proposition 1, we have, with a probability $1 - 3n^{-\beta}$, $\N_\X = 0$. Besides, since $\P_{\Omega}(\N_\X+\N_\E) = \P_{\Omega}(\N_\E) = 0$ and $\Delta \subset \Omega$, we have $\P_{\Delta}(\N_\E) = 0$. Since $\N_\E = \P_{\Delta}(\N_\E) + \P_{\Delta^c}(\N_\E)$, we have $\N_\E = 0$, which leads to the contradiction.
\end{proof}

Given Theorem~\ref{thm:1}, we are now ready to prove Theorem 3.1.
\begin{proof}
The key to the proof is to construct the matrix $\Q$ that satisfies the conditions (a)-(e) specified in Theorem~\ref{thm:1}. First, according to Theorem~\ref{thm:1}, when $m_1 - m_2 > m_0=C_R^2\mu_0rn\beta\log n$, with a probability at least $1 - 3n^{-\beta}$, mapping $\P_T\P_{\Omega}\P_T(\Z): T \mapsto T$ is an one to one mapping and therefore its inverse mapping, denoted by $(\P_T\P_{\Omega}\P_T)^{-1}$ is well defined. Similar to the proof of Theorem 2 in~\cite{chandrasekaran2011rank}, we construct the dual certificate $\Q$ as follows
\[
    \Q = \lambda\ \sgn(\E) + \epsilon_{\Delta} + \P_{\Delta}\P_T(\P_T\P_{\Omega}\P_T)^{-1}(\U\V^{\top} + \epsilon_T)
\]
where $\epsilon_T \in T$ and $\epsilon_{\Delta} = \P_{\Delta}(\epsilon_{\Delta})$. We further define
\begin{eqnarray*}
\H & = & \P_{\Omega}\P_T(\P_T\P_{\Omega}\P_T)^{-1}(\U\V^{\top}) \\
\F & = & \P_{\Omega}\P_T(\P_T\P_{\Omega}\P_T)^{-1}(\epsilon_{T})
\end{eqnarray*}
Evidently, we have $\P_{\Omega}(\Q) = \Q$ since $\Delta \subset \Omega$, and therefore the condition (a) is satisfied. To satisfy the conditions (b)-(e), we need
\begin{eqnarray}
\P_T(\Q) = \U\V^{\top} & \rightarrow & \epsilon_T = -\P_T(\lambda\ \sgn(\E) + \epsilon_{\Delta}) \label{eqn:c1}\\
\|\P_{T^{\perp}}(\Q)\| < 1 & \rightarrow & \mu(\E)\left(\lambda + \|\epsilon_{\Delta}\|_{\infty}\right) + \|\P_{T^{\perp}}(\H)\| + \|\P_{T^{\perp}}(\F)\| < 1\label{eqn:c2} \\
\P_{\Delta}(\Q) = \lambda\ \sgn(\E) & \rightarrow & \epsilon_{\Delta} = - \P_{\Delta}(\H + \F) \label{eqn:c3} \\
|\P_{\Delta^c}(\Q)|_{\infty} < \lambda & \rightarrow & \xi(\X)(1 + \|\epsilon_{T}\|) < \lambda \label{eqn:c4}
\end{eqnarray}
Below, we will first show that there exist solutions $\epsilon_T \in T$ and $\epsilon_{\Delta}$ that satisfy conditions (\ref{eqn:c1}) and (\ref{eqn:c3}). We will then bound $\|\epsilon_{\Omega}\|_{\infty}$, $\|\epsilon_T\|$, $\|\P_{T^{\perp}}(\H)\|$, and $\|\P_{T^{\perp}}(\F)\|$ to show that with sufficiently small $\mu(\E)$ and $\xi(\X)$, and appropriately chosen $\lambda$, conditions (\ref{eqn:c2}) and (\ref{eqn:c4}) can be satisfied as well.

First, we show the existence of $\epsilon_{\Delta}$ and $\epsilon_T$ that obey the relationships in (\ref{eqn:c1}) and (\ref{eqn:c3}). It is equivalent to show that there exists $\epsilon_T$ that satisfies the following relation
\[
\epsilon_T = -\P_T(\lambda\ \sgn(\E)) + \P_T\P_{\Delta}(\H) + \P_T\P_{\Delta}\P_T(\P_T\P_{\Omega}\P_T)^{-1}(\epsilon_T)
\]
or
\[
\P_T\P_{\Omega\setminus\Delta}\P_T(\P_T\P_{\Omega}\P_T)^{-1}(\epsilon_T) = -\P_T(\lambda\ \sgn(\E)) + \P_T\P_{\Delta}(\H),
\]
where $\Omega\setminus\Delta$ indicates the complement set of set $\Delta$ in $\Omega$ and $|\Omega\setminus\Delta|$ denotes its cardinality. Similar to the previous argument, when $|\Omega\setminus\Delta| = m_1 - m_2 > m_0$, with a probability $1 - 3n^{-\beta}$, $\P_T\P_{\Omega\setminus\Delta}\P_T(\Z): T \mapsto T$ is an one to one mapping, and therefore $(\P_T\P_{\Omega\setminus\Delta}\P_T(\Z))^{-1}$ is well defined. Using this result, we have the following solution to the above equation
\begin{eqnarray*}
\epsilon_T & = & \P_T\P_{\Omega}\P_T(\P_T\P_{\Omega\setminus\Delta}\P_T)^{-1}\left(-\P_T(\lambda\ \sgn(\E)) + \P_T\P_{\Delta}(\H) \right)
\end{eqnarray*}

We now bound $\|\epsilon_T\|$ and $\|\epsilon_{\Delta}\|_{\infty}$. Since $\|\epsilon_T\| \leq \|\epsilon_T\|_F$, we bound $\|\epsilon_T\|_F$ instead. First, according to Corollary 3.5 in~\cite{candes2010power}, when $\beta = 4$, with a probability $1 - n^{-3}$, for any $\Z \in T$, we have
\[
    \left\|\P_{T^{\perp}}\P_{\Omega}\P_T(\P_T\P_{\Omega}\P_T)^{-1}(\Z)\right\|_F \leq \|\Z\|_F.
\]
Using this result, we have
\begin{eqnarray*}
    \|\epsilon_{\Delta}\|_{\infty} & \leq & \xi(\X)\left(\|\H\| + \|\F\|\right) \\
    & \leq & \xi(\X)\left(1 + \|\P_{T^{\perp}}(\H)\|_F + \|\epsilon_T\| + \|\P_{T^{\perp}}(\F)\|_F\right) \\
    & \leq & \xi(\X)\left(2 + \|\epsilon_T\| + \|\epsilon_T\|_F\right) \\
    & \leq & \xi(\X)\left[2 + (2k+1)\|\epsilon_T\|\right]
\end{eqnarray*}
In the last step, we use the fact that $\rank(\epsilon_T) \leq 2k$ if $\epsilon_T \in T$. We then proceed to bound $\|\epsilon_T\|$ as follows
\begin{eqnarray*}
\|\epsilon_T\| & \leq & \mu(\E)\left(\lambda + \|\epsilon_{\Delta}\|_{\infty}\right)
\end{eqnarray*}
Combining the above two inequalities together, we have
\begin{eqnarray*}
\|\epsilon_T\| & \leq & \xi(\X)\mu(\E)(2k + 1)\|\epsilon_T\| + 2\xi(\X)\mu(\E) + \lambda \mu(\E) \\
\|\epsilon_{\Delta}\|_{\infty} & \leq & \xi(\X)\left[2 + (2k+1)\mu(\E)(\lambda + \|\epsilon_{\Delta}\|_{\infty}\right),
\end{eqnarray*}
which lead to
\begin{eqnarray*}
\|\epsilon_T\| & \leq & \frac{\lambda \mu(\E) + 2\xi(\X)\mu(\E)}{1 - (2k+1)\xi(\X)\mu(\E)} \\
\|\epsilon_{\Delta}\|_{\infty} & \leq & \frac{2\xi(\X) + (2k+1)\lambda\xi(\X)\mu(\E)}{1 - (2k+1)\xi(\X)\mu(\E)}
\end{eqnarray*}
Using the bound for $\|\epsilon_{\Delta}\|_{\infty}$ and $\|\epsilon_T\|$, we now check the condition (\ref{eqn:c2})
\begin{eqnarray*}
1 & > & \mu(\E)\left(\lambda + |\epsilon_{\Delta}|_{\infty}\right) + \frac{1}{2} + \frac{k}{2}\|\epsilon_T\|
\end{eqnarray*}
or
\[
\lambda < \frac{1 - \xi(\X)\mu(\E)(4k + 5)}{\mu(\E)(k+2)}
\]
For the condition (\ref{eqn:c4}), we have
\[
    \lambda > \xi(\X) + \xi(\X)\|\epsilon_T\|
\]
or
\[
    \lambda > \frac{\xi(\X) - (2k - 1)\xi^2(\X)\mu(\E)}{1 - 2(k+1)\xi(\X)\mu(\E)}
\]
To ensure that there exists $\lambda \geq 0$ satisfies the above two conditions, we have
\[
1 - 5(k+1)\xi(\X)\mu(\E) + (10k^2+21k+8)[\xi(X)\mu(\E)]^2 > 0
\]
and
\[
1 - \xi(\X)\mu(\E)(4k + 5) \geq 0
\]
Since the first condition is guaranteed to be satisfied for $k \geq 1$, we have
\[
    \xi(\X)\mu(\E) \leq \frac{1}{4k + 5}.
\]
Thus we finish the proof.
\end{proof}

\section*{Appendix C: Data Statistics}
We listed the detailed domains of the sentiment analysis tasks in Table \ref{tab:sa_data}. We removed the \emph{musical\_instruments} and \emph{tools\_hardware} domains from the original data because they have too few labeled examples. The statistics for the 10 target tasks of intent classification in Table \ref{tab:nlu_data}

\begin{table*}[th]
\centering
\small
\caption{Statistics of the Multi-Domain Sentiment Classification Data.}\label{tab:sa_data}
\vspace{0.1 in}
\begin{tabular}{|c|c|c|c|}
\hline
{\bf Domains} & {\bf \#train} &{\bf \#validation}  & {\bf \#test}  \\ \hline
apparel & 7398 & 926 & 928 \\
automotive & 601 & 69 & 66 \\
baby & 3405 & 437 & 414 \\
beauty & 2305 & 280 & 299 \\
books & 19913 & 2436 & 2489 \\
camera\_photo & 5915 & 744 & 749 \\
cell\_phones\_service & 816 & 109 & 98 \\
computer\_video\_games & 2201 & 274 & 296 \\
dvd & 19961 & 2624 & 2412 \\
electronics & 18431 & 2304 & 2274 \\
gourmet\_food & 1227 & 182 & 166 \\
grocery & 2101 & 268 & 263 \\
health\_personal\_care & 5826 & 687 & 712 \\
jewelry\_watches & 1597 & 188 & 196 \\
kitchen\_housewares & 15888 & 1978 & 1990 \\
magazines & 3341 & 427 & 421 \\
music & 20103 & 2463 & 2510 \\
office\_products & 337 & 54 & 40 \\
outdoor\_living & 1321 & 143 & 135 \\
software & 1934 & 254 & 202 \\
sports\_outdoors & 4582 & 566 & 580 \\
toys\_games & 10634 & 1267 & 1246 \\
video & 19941 & 2519 & 2539 \\
\hline
\end{tabular}
\end{table*}

\begin{table*}[th]
\centering
\small
\caption{Statistics of the User Intent Classification Data.}\label{tab:nlu_data}
\vspace{0.1 in}
\begin{tabular}{|c|c|c|}
\hline
{\bf Dataset ID} & {\bf \#labeled instances} &{\bf \#labels}   \\ \hline
1	&	497	&	11	\\
2	&	3071	&	14	\\
3	&	305	&	21	\\
4	&	122	&	7	\\
5	&	110	&	11	\\
6	&	126	&	12	\\
7	&	218	&	45	\\
8	&	297	&	10	\\
9	&	424	&	4	\\
10	&	110	&	17	\\
\hline
\end{tabular}
\end{table*}


\end{document}